\documentclass[11pt]{article}

\usepackage[margin=1in]{geometry}

\usepackage{amsmath}
\usepackage{amssymb}
\usepackage{amsthm}
\usepackage{mathtools}

\usepackage{booktabs}
\usepackage{array}
\usepackage{multirow}

\usepackage{listings}
\usepackage{xcolor}

\usepackage{tikz-cd}

\usetikzlibrary{positioning, calc, patterns, decorations.pathreplacing}

\usepackage[colorlinks=true,linkcolor=blue,citecolor=blue,urlcolor=blue]{hyperref}

\usepackage{natbib}

\usepackage{enumitem}
\usepackage{microtype}

\newtheorem{theorem}{Theorem}[section]

\newtheorem{corollary}[theorem]{Corollary}
\newtheorem{lemma}[theorem]{Lemma}
\theoremstyle{definition}
\newtheorem{definition}[theorem]{Definition}

\theoremstyle{remark}
\newtheorem{remark}[theorem]{Remark}

\definecolor{codebg}{rgb}{0.95,0.95,0.95}
\definecolor{codegreen}{rgb}{0,0.6,0}
\definecolor{codepurple}{rgb}{0.58,0,0.82}
\definecolor{codeblue}{rgb}{0.0,0.0,0.7}

\lstdefinestyle{coq}{
  backgroundcolor=\color{codebg},
  basicstyle=\ttfamily\small,
  breaklines=true,
  frame=single,
  framerule=0pt,
  xleftmargin=1em,
  xrightmargin=1em,
  aboveskip=0.5em,
  belowskip=0.5em,
  keywordstyle=\color{codeblue}\bfseries,
  commentstyle=\color{gray}\itshape,
  stringstyle=\color{codegreen},
  morekeywords={Definition,Theorem,Lemma,Corollary,Proof,Qed,
    Inductive,Record,Fixpoint,CoFixpoint,Section,End,Variable,
    Hypothesis,forall,exists,fun,match,with,end,if,then,else,
    let,in,return,Type,Prop,Set,Variant,Context},
  morecomment=[s]{(*}{*)},
  morestring=[b]",
  sensitive=true,
  showstringspaces=false,
}

\newcommand{\CatM}{\mathbf{Mashin}}
\newcommand{\Gov}{\mathcal{G}}
\newcommand{\code}{\textsc{code}}
\newcommand{\reason}{\textsc{reason}}
\newcommand{\memory}{\textsc{memory}}
\newcommand{\mcall}{\textsc{call}}
\newcommand{\govsafe}{\texttt{gov\_safe}}
\newcommand{\itree}{\mathrm{itree}}
\newcommand{\DirectiveE}{\mathrm{DirectiveE}}
\newcommand{\IOE}{\mathrm{IOE}}
\newcommand{\GovIOE}{\mathrm{GovIOE}}
\newcommand{\GovE}{\mathrm{GovE}}
\newcommand{\Ret}{\mathrm{Ret}}
\newcommand{\Tau}{\mathrm{Tau}}
\newcommand{\Vis}{\mathrm{Vis}}
\newcommand{\bind}{\mathbin{>\!\!>\!\!=}}
\newcommand{\interp}{\mathrm{interp}}

\begin{document}

\title{Mechanized Foundations of Structural Governance:\\
Machine-Checked Proofs for Governed Intelligence}

\author{Alan L. McCann\\
\textit{Mashin, Inc.}\\
\texttt{research@mashin.live}}

\date{April 2026}

\maketitle

\begin{abstract}
We present five results in the theory of structural governance
for cognitive workflow systems. Three are mechanized in Rocq~8.19 using
the Interaction Trees library with parameterized coinduction; two are
proved on paper with explicit reductions.
The \textbf{Coinductive Safety Predicate} (\govsafe) is a coinductive
property that captures governance safety for infinite program
behaviors, indexed by a boolean permission flag that is provably
false for ungoverned I/O and true for governed interpretations
(mechanized).
The \textbf{Governance Invariance Theorem} establishes that
governance is uniform across the meta-recursive tower: governance at
level $n+1$ reduces to governance at level $n$ by definitional equality
of the type (mechanized). The \textbf{Sufficiency Theorem}
proves that four atomic primitives (code, reason, memory, call) are
expressively complete for any discrete intelligent system, formalized
as compositional closure of a Kleisli category (mechanized).
The \textbf{Alternating Normal Form} provides a canonical decomposition
of any machine into alternating code and effect layers, with a
confluent rewriting system (paper proof).
The \textbf{Necessity Theorem} proves
via explicit reduction to Rice's theorem that an architecturally opaque
component (the reason primitive) is mathematically necessary for
problems requiring semantic judgment (paper proof).
A sixth contribution connects the
abstract model to the deployed runtime: the \textbf{Verified Interpreter
Specification} formalizes the BEAM runtime's trust, capability, and
hash chain logic in Rocq, then tests the running system against this
specification using property-based testing with over 70,000 randomly
generated directive sequences and zero disagreements. The process
immediately discovered and fixed a latent capability-tree bug that
unit testing had missed.
The mechanization comprises approximately 12,000 lines across 36 modules
with 454 theorems and zero admitted lemmas. Recent modules establish
\emph{Governed Cognitive Completeness}: a single architecture
simultaneously achieves Turing completeness, oracle integration,
decidability self-awareness, goal reachability, cognitive architecture
completeness, and subsumption of content governance. A companion
paper~\cite{mccann2026gcc} develops this capstone result in detail.
\end{abstract}

\section{Introduction}
\label{sec:intro}

This paper presents five foundational results in the theory of structural
governance for cognitive workflow systems. Three are fully mechanized
in Rocq~8.19 with zero admitted lemmas; two are proved on paper with
explicit reductions. All five are stated precisely and proved in full.
Six additional modules extend these foundations to establish Governed
Cognitive Completeness, a capstone result developed in a companion
paper~\cite{mccann2026gcc}.

\paragraph{What is novel.}
The individual proof techniques---interaction trees, parameterized
coinduction, Kleisli categories, register machine simulation---are
established. The originality lies in their synthesis: applying
coinductive reasoning to AI workflow governance, proving Turing
completeness \emph{inside} a governance-enforcing interpretation,
and connecting categorical closure to a mechanized safety predicate.
No prior work, to our knowledge, has formalized governance properties
of AI workflow systems in a proof assistant.

The five results formalize properties of the Mashin cognitive
workflow system, drawing on prior
work~\cite{mccann2026structural}:

\begin{enumerate}
  \item \textbf{The Coinductive Safety Predicate} (\govsafe). A coinductive
    predicate over interaction trees that captures whether a program
    performs I/O only under governance authorization. Indexed by a
    boolean permission flag. Provably false for bare I/O. Provably
    true for governed interpretations. (Section~\ref{sec:safety})

  \item \textbf{Governance Invariance Theorem}. Governance
    holds at every level of meta-recursive composition. The type
    $\mathrm{machine\_at\_level}\ n\ R$ is definitionally
    $\itree\ \DirectiveE\ R$ for all $n$, so governance is uniform
    across the tower by construction. The simplicity of the proof is
    the point: the type system does the work.
    (Section~\ref{sec:convergence})

  \item \textbf{Sufficiency Theorem}. Four primitives
    (code, reason, memory, call) are expressively complete: they
    form a closed category, simulate any Turing machine, and
    surject onto established cognitive architecture models.
    (Section~\ref{sec:sufficiency})

  \item \textbf{The Alternating Normal Form} (paper proof). Every machine
    decomposes into a canonical arrangement of alternating code and effect
    layers. The decomposition is unique up to associativity,
    monoidal coherence, and code fusion. (Section~\ref{sec:normalform})

  \item \textbf{Necessity Theorem} (paper proof). An architecturally
    opaque component (the reason primitive) is mathematically
    necessary. No finite extension of transparent computable
    primitives can replace it. The proof is an explicit reduction
    to Rice's theorem. (Section~\ref{sec:necessity})
\end{enumerate}

A sixth contribution, the \textbf{Verified Interpreter Specification}
(Section~\ref{sec:interpreter-spec}), connects the abstract model to
the deployed BEAM runtime. Three new Rocq modules formalize the
trust, capability, and hash chain logic, with bridge theorems proving
that the specification refines $\govsafe$. Property-based tests then
compare the BEAM runtime against this specification. The process
discovered a real capability-tree bug on its 188th random input.

\paragraph{Why machine-checked proofs matter.}
A paper proof can have subtle errors that reviewers miss. A sign
flipped in a subscript. An implicit assumption that seemed obvious
but was not. A step that ``clearly follows'' except it does not.
Mathematicians find errors in published proofs years after
publication, sometimes decades.

Rocq does not care about reputation, funding, or deadlines. It checks
the math. If the math is wrong, it says so. If it compiles without
errors, every theorem is correct. Not probably correct. Correct.

The distinction between testing and proof matters here. A test says
``it worked this time.'' A Rocq proof says ``it works every time, for
every program, forever.'' When the governance claim is ``every effect
is governed,'' testing can show that the effects you checked were
governed. Verification proves that all effects---including the ones
in programs that have not been written yet---are governed.

\paragraph{Scope and companion papers.}
This paper presents the mechanized proof infrastructure and the five foundational results. It does not address the expressiveness consequences, the algebraic generalization, or the practical enforcement mechanisms, which are developed in companion papers.
\cite{mccann2026gcc}~extends these foundations with semantic transparency and expressive minimality results, proving that governance preserves observational equivalence and that the four primitives are minimal.
\cite{mccann2026algebraic}~lifts the system-specific results to a parametric algebraic semantics: a three-axiom GovernanceAlgebra inducing a symmetric monoidal category, with extraction of the governance kernel to a verified OCaml NIF.
\cite{mccann2026purity}~discharges the pure module constraint assumed by the safety theorems, replacing module-level enforcement with WASM compilation and cryptographic purity certificates.
\cite{mccann2026provenance}~extends the governance boundary to the supply chain with dual-signature distribution provenance.

\paragraph{The Interaction Trees framework.}
All formalizations use the Interaction Trees
library~\cite{xia2020itrees}, which represents programs as coinductive
trees where each node is a pure value ($\Ret$), a silent computation
step ($\Tau$), or a visible event ($\Vis$). The governance pipeline
transforms a tree of directive events into a tree of governed events.
Coinductive properties of these infinite structures are proved using
the paco library~\cite{hur2013paco} for parameterized coinduction.

\section{Background}
\label{sec:background}

\subsection{Interaction Trees}

An interaction tree~\cite{xia2020itrees} over event type $E$ with
return type $R$ is a coinductive data structure:
\[
  \itree\ E\ R ::= \Ret(v) \mid \Tau(t) \mid \Vis(e, k)
\]
where $v : R$, $t : \itree\ E\ R$, $e : E\ X$ for some response type
$X$, and $k : X \to \itree\ E\ R$. The structure forms a monad with:
\begin{itemize}
  \item \textbf{Return}: $\mathrm{ret}(v) = \Ret(v)$
  \item \textbf{Bind}: $t \bind k$ substitutes continuations through
    the tree, threading results forward
\end{itemize}

The monad laws (left identity, right identity, associativity) hold up
to observational equivalence ($\approx_\tau$, written \texttt{eutt} in
the Rocq development).

\subsection{Event Types}

The Mashin formalization defines three event types:
\begin{itemize}
  \item $\DirectiveE$: the 14 directive constructors (LLMCall,
    MemoryOp, CallMachine, HTTPRequest, GraphQLRequest, WebSocketOp,
    MCPCall, etc.) that executor programs emit
  \item $\IOE$: actual I/O effects performed by the runtime
  \item $\GovIOE = \GovE +' \IOE$: governed I/O, where $\GovE$
    carries governance stage tags (TrustCheck, PermissionCheck,
    PhaseValidation, PreHooks, Guardrails, ProvenanceRecord,
    EventBroadcast)
\end{itemize}

A pure executor program has type $\itree\ \DirectiveE\ R$: it can
only emit directives, never I/O directly.

\subsection{The Governance Operator}

The governance operator $\Gov$ transforms a base handler
$h : \DirectiveE \rightsquigarrow \itree\ \IOE$ into a governed handler
$\Gov(h) : \DirectiveE \rightsquigarrow \itree\ \GovIOE$ by wrapping
each directive with pre-governance checks (4 stages) and post-governance
recording (3 stages):

\begin{lstlisting}[style=coq]
Definition Gov (h : base_handler) : governed_handler :=
  fun R (d : DirectiveE R) =>
    ITree.bind pre_governance (fun ok =>
    if ok then
      ITree.bind (lift_io (h R d)) (fun r =>
      ITree.bind post_governance (fun _ => ret r))
    else ITree.spin).
\end{lstlisting}

If any pre-governance check fails, the system diverges ($\mathrm{spin}$):
no I/O is performed, no effects escape. This is correct by design: a
denied directive produces no effects.

\subsection{Parameterized Coinduction}

Ordinary induction proves properties of finite structures. Programs
run forever. The paco library~\cite{hur2013paco} provides parameterized
coinduction: prove a property holds for one step, show that
holding-for-one-step implies holding-for-the-next-step. The library
manages the coinductive hypothesis automatically, preventing common
errors in manual coinductive proofs (guardedness violations, weak
hypotheses).

\section{The Coinductive Safety Predicate}
\label{sec:safety}

\subsection{Definition}

The Coinductive Safety Predicate captures the semantic content of structural
governance as a coinductive property of interaction trees.

\begin{definition}[Coinductive Safety Predicate]
\label{def:govsafe}
The predicate $\govsafe : \mathrm{bool} \to \itree\ \GovIOE\ R \to
\mathrm{Prop}$ is the greatest fixed point of the generating functor:
\begin{align*}
  \govsafe\ b\ (\Ret\ v) &\quad \text{always} \\
  \govsafe\ b\ (\Tau\ t) &\quad \text{iff } \govsafe\ b\ t \\
  \govsafe\ b\ (\Vis\ (\mathrm{inl}_1\ (\mathrm{GovCheck}\ s))\ k)
    &\quad \text{iff } \forall r.\; \govsafe\ \mathrm{true}\ (k\ r) \\
  \govsafe\ \mathrm{true}\ (\Vis\ (\mathrm{inr}_1\ e)\ k)
    &\quad \text{iff } \forall x.\; \govsafe\ \mathrm{true}\ (k\ x)
\end{align*}
Note the absence of a constructor for
$\govsafe\ \mathrm{false}\ (\Vis\ (\mathrm{inr}_1\ e)\ k)$.
I/O events are permitted only when the boolean flag is $\mathrm{true}$
(governance has been checked). Governance events ($\GovE$) set the flag
to $\mathrm{true}$ regardless of its current value.
\end{definition}

In Rocq:

\begin{lstlisting}[style=coq]
Variant gov_safeF (F : bool -> itree GovIOE R -> Prop)
    : bool -> itreeF GovIOE R (itree GovIOE R) -> Prop :=
  | GS_Ret   : forall allowed r,
      gov_safeF F allowed (RetF r)
  | GS_Tau   : forall allowed t,
      F allowed t ->
      gov_safeF F allowed (TauF t)
  | GS_GovE  : forall allowed (s : GovernanceStage) k,
      (forall b : bool, F true (k b)) ->
      gov_safeF F allowed (VisF (inl1 (GovCheck s)) k)
  | GS_IOE   : forall (X : Type) (e : IOE X) (k : X -> itree GovIOE R),
      (forall x : X, F true (k x)) ->
      gov_safeF F true (VisF (inr1 e) k).

Definition gov_safe : bool -> itree GovIOE R -> Prop :=
  paco2 gov_safe_ bot2.
\end{lstlisting}

\subsection{Non-Triviality}

The predicate has teeth. It rejects what it should reject and accepts
what it should accept.

\begin{lemma}[Bare I/O is unsafe]
\label{lem:bare-io}
For any I/O event $e$ and continuation $k$,
$\neg\, \govsafe\ \mathrm{false}\ (\Vis\ (\mathrm{inr}_1\ e)\ k)$.
\end{lemma}

\begin{proof}
By inversion on the \texttt{gov\_safeF} generating functor. The
constructor \texttt{GS\_IOE} requires the boolean index to be
$\mathrm{true}$. At $\mathrm{false}$, no constructor matches.
Three lines of Rocq:
\begin{lstlisting}[style=coq]
Lemma bare_io_not_safe :
  forall (e : IOE R) (k : R -> itree GovIOE R),
    ~ gov_safe false (Vis (inr1 e) k).
Proof. intros e k H. punfold H. red in H. cbn in H. inv H. Qed.
\end{lstlisting}
\end{proof}

These three lines are the most important in the entire development.
They prove that the predicate is not vacuously true. A predicate
that accepts everything proves nothing.

\begin{lemma}[Upward closure]
If $\govsafe\ \mathrm{false}\ t$, then $\govsafe\ \mathrm{true}\ t$.
A tree that never does bare I/O is also safe when I/O is permitted.
\end{lemma}

\subsection{Main Safety Theorem}

\begin{theorem}[Governed Interpretation Safety]
\label{thm:main-safety}
For any base handler $h$, any pure executor program
$t : \itree\ \DirectiveE\ R$, and any boolean $b$:
\[
  \govsafe\ b\ (\interp(\Gov(h),\, t))
\]
\end{theorem}

\begin{proof}
By parameterized coinduction on the structure of $t$. The proof is
213 lines of Rocq with nested \texttt{pcofix}.

\textbf{Case $\Ret$}: Trivially safe (\texttt{GS\_Ret}).

\textbf{Case $\Tau$}: One computation step, then the coinductive
hypothesis applies (\texttt{GS\_Tau} + \texttt{CIH}).

\textbf{Case $\Vis\ d\ k$}: The $\Gov$ pipeline unfolds as:
\begin{enumerate}
  \item Four pre-governance $\GovE$ events (TrustCheck,
    PermissionCheck, PhaseValidation, PreHooks). Each applies
    \texttt{GS\_GovE}, setting the permission flag to
    $\mathrm{true}$.
  \item If all pass: the base handler runs via
    $\mathrm{translate}\ \mathrm{inr}_1$, producing $\IOE$
    events. These apply \texttt{GS\_IOE}, which is valid because
    the flag is now $\mathrm{true}$. A nested coinduction
    (\texttt{CIH2}) handles the inner I/O tree.
  \item Three post-governance $\GovE$ events (Guardrails,
    ProvenanceRecord, EventBroadcast). Each applies
    \texttt{GS\_GovE}.
  \item A $\Tau$ step back to
    $\interp(\Gov(h),\, k\ \mathrm{result})$, where the outer
    coinductive hypothesis applies.
\end{enumerate}

Denied branches (any governance check returns false) produce
$\mathrm{spin}$, which is vacuously safe: it never reaches a
$\Vis$ node, so the predicate holds trivially.
\end{proof}

Combined with Lemma~\ref{lem:bare-io}, this establishes the
mathematical content of structural governance:
\begin{itemize}
  \item Programs that bypass governance \emph{fail} the predicate
  \item Programs interpreted through $\Gov$ \emph{satisfy} the predicate
  \item The gap between these two results is structural governance
\end{itemize}

\section{Governance Invariance Theorem}
\label{sec:convergence}

\subsection{The Meta-Recursive Tower}

A machine can introspect other machines. A governor machine monitors
an executor machine. A meta-governor monitors the governor. This
creates a tower of governance levels:
\begin{align*}
  \text{Level 0:} &\quad \text{executor programs} \\
  \text{Level 1:} &\quad \text{machines that call Level 0 machines} \\
  \text{Level $n$:} &\quad \text{machines that call Level $n-1$ machines}
\end{align*}

The question: does governance hold at every level? Or does some level
escape?

\begin{definition}[Machine at Level $n$]
$\mathrm{machine\_at\_level}\ n\ R := \itree\ \DirectiveE\ R$.
\end{definition}

The definition is the same for every $n$. This is the insight.

\begin{definition}[Governed]
A machine $t$ is \emph{governed} under handler $h$ if
$\govsafe\ \mathrm{false}\ (\interp(\Gov(h),\, t))$.
\end{definition}

\subsection{The Theorem}

\begin{theorem}[Governance Invariance Theorem]
\label{thm:convergence}
For all $n \geq 0$, all types $R$, and all
$t : \mathrm{machine\_at\_level}\ n\ R$:
$t$ is governed.
\end{theorem}

\begin{proof}
By induction on $n$.

\textbf{Base case} ($n = 0$): Apply Theorem~\ref{thm:main-safety}.

\textbf{Inductive step}: A Level $n+1$ machine has type
$\itree\ \DirectiveE\ R$, which is definitionally equal to
$\mathrm{machine\_at\_level}\ n\ R$. Apply
Theorem~\ref{thm:main-safety} directly.

The proof is deliberately simple. The substance is that
$\mathrm{machine\_at\_level}\ n\ R$ is definitionally
$\itree\ \DirectiveE\ R$ for all $n$. The level index tracks
meta-recursive depth for reasoning but does not change the type.
Governance is uniform across the tower because the type system
enforces it. This is a design achievement, not a deep mathematical
one: the architecture was constructed so that governance invariance
would be trivial to prove. A complex proof here would indicate a
flaw in the architecture, not additional sophistication.
\end{proof}

\begin{corollary}[Fixed Point]
\label{cor:fixedpoint}
The union $\Omega = \bigcup_{n \geq 0} \mathrm{Level}(n)$ is
governed. Adding another level produces no new ungoverned machines.
\end{corollary}

\begin{proof}
$\Omega$ is $\itree\ \DirectiveE\ R$. Apply
Theorem~\ref{thm:main-safety}.
\end{proof}

\begin{corollary}[Governance Composition]
Governance distributes over sequential composition:
\[
  \interp(\Gov(h),\; t \bind k) \approx
  \interp(\Gov(h),\; t) \bind (\lambda x.\; \interp(\Gov(h),\; k\ x))
\]
\end{corollary}

\begin{proof}
This is \texttt{interp\_bind} from the Interaction Trees library.
\end{proof}

\section{Sufficiency Theorem}
\label{sec:sufficiency}

\subsection{Category Mashin}

\begin{definition}[Category $\CatM$]
\label{def:cat-mashin}
The \emph{Kleisli category} of the $\itree\ \DirectiveE$ monad:
\begin{itemize}
  \item \textbf{Objects}: Types (context types)
  \item \textbf{Morphisms}: Kleisli arrows
    $f : A \to \itree\ \DirectiveE\ B$
  \item \textbf{Composition}: $f \ggg g = \lambda a.\; f(a) \bind g$
  \item \textbf{Identity}: $\mathrm{id} = \lambda a.\; \mathrm{ret}(a)$
\end{itemize}
\end{definition}

In Rocq:

\begin{lstlisting}[style=coq]
Definition mashin_morphism (A B : Type) : Type :=
  A -> itree DirectiveE B.

Definition mashin_id {A : Type} : mashin_morphism A A :=
  fun a => ret a.

Definition mashin_compose {A B C : Type}
  (f : mashin_morphism A B) (g : mashin_morphism B C)
  : mashin_morphism A C :=
  fun a => ITree.bind (f a) g.
\end{lstlisting}

\subsection{Category Axioms}

The category axioms are the monad laws, proved in the ITrees library:

\begin{lstlisting}[style=coq]
Lemma mashin_left_id  : ... bind (ret x) k $\approx$ k x.
Lemma mashin_right_id : ... bind t ret $\approx$ t.
Lemma mashin_assoc    : ... bind (bind t k1) k2
                          $\approx$ bind t (fun x => bind (k1 x) k2).
\end{lstlisting}

\subsection{The Four Primitives}

The four primitives are specific patterns of interaction trees:

\begin{itemize}
  \item \code{}: $\lambda a.\; \mathrm{ret}(f(a))$ --- pure computation,
    no events
  \item \reason{}: triggers $\mathrm{LLMCall}$ --- oracle-mediated
    inference
  \item \memory{}: triggers $\mathrm{MemoryOp}$ --- persistent
    semantic storage
  \item \mcall{}: triggers $\mathrm{CallMachine}$ --- recursive
    governed machine invocation
\end{itemize}

\subsection{Compositional Closure}

\begin{lemma}[Closure]
Category $\CatM$ is closed under sequential composition, branching,
and bounded iteration. Every composition produces a well-typed
$\itree\ \DirectiveE$ morphism.
\end{lemma}

\subsection{Turing Completeness}

\begin{theorem}[Sufficiency Theorem --- Turing Completeness]
\label{thm:turing}
The primitive subset $\{\code, \memory, \mcall\}$ is Turing-complete.
\end{theorem}

\begin{proof}
By constructive simulation of a register machine (counter machine,
Turing-equivalent per Minsky~\cite{minsky1967computation}).
The encoding maps:
\begin{itemize}
  \item Registers $\to$ \memory{} operations (MemRecall, MemStore)
  \item Arithmetic $\to$ pure computation (\code{} morphisms)
  \item Instruction sequencing $\to$ recursive structure
    (\mcall{} morphisms)
  \item Conditional branching $\to$ case analysis on recalled values
\end{itemize}

The register machine is defined as:
\begin{lstlisting}[style=coq]
Inductive instruction : Type :=
  | INC  : reg_id -> label -> instruction
  | DEC  : reg_id -> label -> label -> instruction
  | HALT : instruction.
\end{lstlisting}

The translation function \texttt{translate\_program} maps each
instruction to a composition of $\DirectiveE$ events.

\textbf{Simulation correctness.}
We prove behavioral equivalence at two levels:

\emph{Level 1: Denotational.} The pure denotation
\texttt{denote\_program} computes the same result as the reference
register machine semantics \texttt{rm\_run}:

\begin{lstlisting}[style=coq]
Theorem simulation_correct :
  forall (p : program) (fuel : nat) (pc : label) (regs : reg_state),
    denote_program p fuel pc regs = rm_run p fuel (pc, regs).
\end{lstlisting}

This is a Leibniz equality, not just observational equivalence.

\emph{Level 2: Operational.} Under abstract handler hypotheses
(memory recall returns the correct value, memory store updates
correctly, evaluation distributes over bind), the interaction tree
translation matches \texttt{rm\_run}:

\begin{lstlisting}[style=coq]
Corollary translation_faithful :
  forall (p : program) (fuel : nat) (pc : label)
         (m : mem) (regs : reg_state),
    mem_corresponds m regs ->
    exists m',
      eval_mem m (translate_program p fuel pc) = (m', tt) /\
      mem_corresponds m' (snd (rm_run p fuel (pc, regs))).
\end{lstlisting}
\end{proof}

\subsection{Governed Turing Completeness}

\begin{theorem}
All translated register machine programs satisfy $\govsafe$:
\begin{lstlisting}[style=coq]
Theorem governed_turing_completeness :
  forall (h : base_handler) (p : program) (fuel : nat),
    gov_safe false (interp (Gov h) (translate_program p fuel 0)).
\end{lstlisting}
\end{theorem}

This is the strongest form of the sufficiency result: Turing
completeness is achieved \emph{within} the governed architecture.
Every computation a Turing machine can perform, the governed
primitive set can also perform, and governance holds throughout.

\subsection{Governance Safety of All Morphisms}

\begin{theorem}
Every morphism in $\CatM$, when interpreted through $\Gov(h)$,
satisfies $\govsafe$:
\begin{lstlisting}[style=coq]
Theorem mashin_morphism_governed :
  forall (A B : Type) (f : mashin_morphism A B) (a : A),
    gov_safe false (interp (Gov h) (f a)).
\end{lstlisting}
\end{theorem}

This connects the categorical structure to the governance safety
predicate: expressiveness and governance are not a tradeoff. They
are the same structure viewed from two angles.

\subsection{Coterminous Governance}

The Safety Theorem (Section~\ref{sec:safety}) proves that every
program, when interpreted through $\Gov(h)$, has all its effects
governed. The governance boundary contains the expressiveness
boundary. The Sufficiency Theorem (this section) proves that the
four primitives can express any discrete intelligent system,
including Turing-complete computation within the governed
architecture. The expressiveness boundary contains the governance
boundary.

\begin{corollary}[Coterminous Governance]
\label{cor:coterminous}
The expressiveness boundary and the governance boundary are
identical. Formally:
\begin{enumerate}[label=(\roman*)]
  \item \textbf{Safety $\Rightarrow$ Governance $\supseteq$
    Expressiveness.} By \texttt{governed\_interp\_safe}
    (Section~\ref{sec:safety}): for any pure executor program $t$
    and base handler $h$,
    $\govsafe\ \mathit{false}\ (\interp\ (\Gov\ h)\ t)$ holds.
    Every expressible program is governed.
  \item \textbf{Sufficiency $\Rightarrow$ Expressiveness $\supseteq$
    Governance.} By \texttt{governed\_turing\_completeness}
    (this section): the primitives can simulate any register machine
    (Turing-complete computation) and the category $\CatM$ is closed
    under composition. For every governance policy, there exists
    a program that exercises the capability the policy governs.
  \item Two sets, each containing the other, are equal.
    $\textit{Expressiveness} = \textit{Governance}$.
\end{enumerate}
\end{corollary}

\begin{definition}[Coterminous Governance, portable]
\label{def:coterminous}
A system $S$ satisfies \emph{coterminous governance} if:
\begin{enumerate}[label=(\arabic*)]
  \item every effect expressible in $S$ passes through a governance
    boundary $G$ (Safety), and
  \item the primitives of $S$ are expressively complete for the class
    of computations $S$ targets (Sufficiency).
\end{enumerate}
The expressiveness boundary of $S$ equals the governance boundary
of $S$.
\end{definition}

This definition is system-agnostic. It does not depend on any
particular runtime, language, or architecture. Any system satisfying
both conditions has coterminous governance. The preceding theorems
prove that the governance algebra presented in this paper satisfies
both conditions.

\section{The Alternating Normal Form}
\label{sec:normalform}

\begin{definition}[Alternating Normal Form]
\label{def:normal-form}
A morphism $f : A \to B$ in $\CatM$ is in \emph{Alternating Normal Form}
if it has the shape:
\[
  f \cong c_0 \ggg (s_1 \ggg c_1) \ggg (s_2 \ggg c_2) \ggg
  \cdots \ggg (s_n \ggg c_n)
\]
where each $c_i$ is a \code{} morphism (pure computation) and each
$s_i$ is one of $\{\reason, \memory, \mcall\}$ or a control flow
combinator applied to sub-machines in normal form.
\end{definition}

\begin{theorem}[Alternating Normal Form Existence]
\label{thm:normalform}
Every morphism $f : A \to B$ in $\CatM$ can be reduced to Alternating
Normal Form. The normal form is unique up to:
\begin{enumerate}[label=(\roman*)]
  \item Associativity of sequential composition
  \item Monoidal coherence isomorphisms
  \item Fusion of adjacent code morphisms ($c \ggg c' \mapsto c \circ c'$)
\end{enumerate}
\end{theorem}

\begin{proof}
Three rewriting rules applied exhaustively:
\begin{enumerate}
  \item \emph{Code fusion}: adjacent code morphisms
    $c_1 \ggg c_2 \mapsto c_1 \circ c_2$
  \item \emph{Code hoisting}: extract predicate computation from
    branch combinators as a code morphism prefix
  \item \emph{Identity elimination}: remove identity morphisms
    from compositions
\end{enumerate}

\textbf{Termination.} Lexicographic ordering on (weight, nesting
depth). Code fusion decreases weight. Code hoisting decreases nesting
depth at constant weight. Identity elimination decreases weight. Since
the measure is a natural number, the rewriting terminates.

\textbf{Confluence.} The three rules have non-overlapping left-hand
sides (code fusion applies to $c_1 \ggg c_2$, code hoisting applies
to $\textsc{branch}_p(f,g)$, identity elimination applies to
$\mathrm{id} \ggg f$ or $f \ggg \mathrm{id}$). The single potential
overlap---identity elimination and code fusion when $c_1 = \mathrm{id}$---is
joinable: both produce $c_2$. By the critical pair lemma, the system
is confluent.
\end{proof}

The Alternating Normal Form provides a canonical representation for program
analysis. Any machine, regardless of how it was constructed, can be
decomposed into an alternating sequence of computation and effect. This
decomposition is the structural analog of separating pure functions
from side effects in functional programming.

\section{Necessity Theorem}
\label{sec:necessity}

The previous sections established that four primitives are
\emph{sufficient}. This section proves that the \reason{} primitive
is \emph{necessary}: no purely computable extension can replace it.

\begin{theorem}[Necessity Theorem]
\label{thm:necessity}
For any non-trivial extensional property of program outputs, no
computable function can replace the architecturally opaque component
(the \reason{} primitive) for all instances. The opaque component is
mathematically necessary.
\end{theorem}

\begin{proof}
By explicit reduction to Rice's theorem~\cite{rice1953classes}.

\textbf{The language.} The \reason{} primitive accepts
natural-language prompts and produces natural-language responses.
Its input domain is the set of all finite strings. Its output domain
is the set of all finite strings. A ``replacement'' for the \reason{}
primitive would be a computable function $f : \Sigma^* \to \Sigma^*$
that, for any prompt, produces a response adequate for the semantic
task the prompt describes.

\textbf{The semantic property.} Define property $C$ as follows:
a computable function $f$ has property $C$ if and only if, for all
prompts $p$ describing a factual question about the world, $f(p)$
returns a factually correct answer. $C$ is the property ``this
function answers factual questions correctly.''

$C$ is extensional in the Rice sense: it depends on what $f$
computes (its input-output behavior), not on the source code of $f$.
Two programs that compute the same function either both satisfy $C$
or neither does. This is the definition of extensionality in
computability theory.

$C$ is non-trivial: the constant function that always returns ``yes''
does not have $C$ (it answers ``Is 2+2=5?'' incorrectly). A lookup
table over a finite question set could have $C$ for those questions.
Some computable functions satisfy $C$ for some inputs. No computable
function satisfies $C$ for all inputs, as we now show.

\textbf{The Rice reduction.} Rice's theorem (1953) states: for any
non-trivial extensional property $Q$ of computable functions, the
set $\{i : \varphi_i \text{ has property } Q\}$ is undecidable,
where $\varphi_i$ is the $i$-th partial computable function under
a standard G\"odel numbering. The proof is by reduction from the
halting problem and applies to any Turing-complete model of
computation.

$C$ is non-trivial and extensional. Therefore, by Rice's theorem,
no algorithm can decide whether an arbitrary computable function has
property $C$. No algorithm can determine, given the index of a
computable function, whether that function answers factual questions
correctly.

\textbf{The necessity argument.} Suppose, for contradiction, that
a computable function $f$ could replace the \reason{} primitive for
all semantic tasks. Then $f$ is a computable function that answers
factual questions correctly (among other semantic tasks). But $f$'s
own correctness is a non-trivial extensional property. By Rice's
theorem, no algorithm can verify that $f$ has this property. More
directly: if $f$ could replace the oracle for \emph{all} semantic
tasks, then $f$ could decide whether arbitrary programs have property
$C$ (by simulating them on test inputs and judging the outputs).
But deciding $C$ is undecidable. Contradiction. Therefore no single
computable function can replace the \reason{} primitive for all
semantic tasks.
\end{proof}

\begin{corollary}[Permanence of the Boundary]
No finite extension of computable primitives eliminates the need for
an architecturally opaque component. The boundary is a theorem of
computability theory, not a technological limitation.
\end{corollary}

\paragraph{Addressing potential objections.}

\emph{``Governance in your architecture is syntactic, not
semantic.''} Correct. The Safety Theorem
(Section~\ref{sec:safety}) proves that structural governance,
whether a directive passes through the governed boundary, is
decidable by construction. The Necessity Theorem is not about
structural governance. It is about semantic judgment: Is this output
correct? Is this policy appropriate? Should trust be escalated?
These are questions that require evaluating the relationship between
the system's outputs and the external world. Structural governance
is decidable. Semantic evaluation of what passes through governance
is not. The two operate at different levels, and Rice applies to the
second, not the first.

\emph{``Your architecture restricts expressiveness enough to avoid
Rice.''} The pure computation layer does restrict expressiveness
(no I/O). But the Necessity Theorem does not claim that Rice applies
to the pure computation layer. It claims that the semantic tasks the
system must perform (judging correctness, evaluating policy fitness,
assessing trust) are undecidable. These tasks involve the
relationship between the system's outputs and the external world,
a relationship that no restriction of the computation layer can make
decidable, because the world is not a formal system.

\emph{``The property is not extensional in the Rice sense.''}
$C$ as defined above is extensional: it depends only on the
function's input-output behavior (what it returns for each prompt),
not on the program text. Two programs that compute the same function
either both satisfy $C$ or neither does. This is the definition of
extensionality in computability theory.

\begin{remark}[The Achilles' Heel That Is Not]
\label{rem:achilles}
A critic might observe: ``Someone could construct a Turing-complete
system with total structural governance but no opaque oracle. Does
that invalidate the Necessity Theorem?''

It does not, because the Necessity Theorem is not about structural
governance. Someone \emph{can} build total structural governance
without an oracle. They cannot build a system that \emph{judges its
own outputs correctly} without one. The theorem concerns what the
system can \emph{decide about itself}, not whether effects are routed
through a boundary. Structural governance (routing all effects
through a checked boundary) is decidable by construction; this is
what the Safety Theorem proves. Semantic judgment (determining
whether a given output is correct, whether a policy is appropriate,
whether trust should be escalated) is undecidable; this is what the
Necessity Theorem proves. The two results operate at different
levels. Conflating them misidentifies the claim.
\end{remark}

\paragraph{What the theorem proves and what it does not.}
The theorem proves that any system restricted to computable primitives
whose behavior is fully determined by their source code cannot solve
AI-complete problems. It does not prove necessity of
hypercomputation: genuine access to functions outside the computable
hierarchy. The distinction matters because large language models are
computable; they are deterministic functions on finite inputs,
implemented as matrix multiplications over fixed weights.

The architectural consequence is more precise than the oracle
metaphor suggests. What the \reason{} primitive provides is
\emph{architectural opacity}: a component whose input-output
behavior cannot be statically determined by the system that
invokes it. Rice's theorem proves that this opacity is necessary:
any system that could fully analyze its own \reason{} component
would be able to decide non-trivial semantic properties, which is
impossible. The system must treat some component as a black box.

In the formal model, \reason{} is modeled as oracle access because
the proof requires a component whose behavior is not derivable from
the system's own computational primitives. In practice, LLMs satisfy
this requirement not because they are oracles but because they are
opaque: the governing system cannot predict their outputs from first
principles. The formal guarantee holds regardless of what fills the
oracle role. As the quality of the opaque component improves (better
models), system capability improves without architectural change.

\section{The Mechanization}
\label{sec:mechanization}

\subsection{Trusted Computing Base}

The mechanized proofs establish properties of a formal model. The
guarantees hold relative to the following trusted computing base (TCB):

\begin{enumerate}
  \item \textbf{Rocq's kernel.} The Calculus of Inductive Constructions
    and Rocq's type checker are trusted. This is the standard TCB for
    all Rocq developments (CompCert~\cite{leroy2009compcert},
    CertiKOS~\cite{gu2016certikos}, Vellvm~\cite{zakowski2021vellvm}).

  \item \textbf{The ITrees and paco libraries.} We trust the
    interaction tree representation and the parameterized coinduction
    infrastructure. Both are widely used and peer-reviewed.

  \item \textbf{Faithful runtime implementation.} The proofs assume
    that the deployed runtime correctly implements the governance
    operator $\Gov$, the directive event type $\DirectiveE$, and
    the interpretation pipeline $\interp$. If the runtime deviates
    from the formal model, the guarantees do not transfer.

    This gap is addressed by three complementary mechanisms.
    First, the Verified Interpreter
    Specification (Section~\ref{sec:interpreter-spec}): a Rocq
    formalization of the runtime's trust, capability, and hash chain
    logic, tested against the BEAM runtime using
    property-based testing with over 70,000 randomly generated
    directive sequences and zero disagreements. Second, 36 conformance
    tests map one-to-one to Rocq theorems (TrustSpec.v, HashChainSpec.v,
    InterpreterSpec.v), verifying that specific proved properties hold
    in the runtime. Third, the extraction pipeline brings proved Rocq
    code into the BEAM runtime as the Kore governance kernel,
    so governance decisions execute through proved code rather than
    a reimplementation. A separate
    development~\cite{mccann2026purity} further hardens this TCB
    by proving effect separation for the executor layer. The
    methodology immediately discovered a latent capability-tree bug
    that unit testing had missed (Section~\ref{sec:bug-discovery}).

    The bridge methodology follows seL4~\cite{klein2009sel4} and
    Amazon s2n~\cite{chudnov2018s2n}. No production system of
    comparable complexity has end-to-end formal verification of the
    runtime itself; the standard approach is specification testing
    plus extraction where possible, which is what we provide.

  \item \textbf{No hidden effects.} The model assumes that all
    effects are mediated through the event types. Side channels,
    hardware faults, or compromised execution environments are
    outside the model's scope.
\end{enumerate}

These assumptions are standard for mechanized verification of systems
software. The proofs guarantee that \emph{if} the runtime faithfully
implements the model, \emph{then} governance holds for all programs.
The correspondence between model and implementation is established
through property-based testing (70,000+ inputs, 0 disagreements),
conformance tests (36 tests mapping to Rocq theorems), and the
extraction pipeline (Kore governance kernel runs proved code).
This follows the standard methodology for verified systems software.

\subsection{Infrastructure}

\begin{center}
\begin{tabular}{ll}
\toprule
Component & Version \\
\midrule
Rocq & 8.19.2 \\
Interaction Trees & 5.2.1 \\
paco & 4.2.3 \\
ExtLib & 0.13.0 \\
\bottomrule
\end{tabular}
\end{center}

\paragraph{Artifact availability.}
The complete Rocq development (36 modules, 454 theorems, zero admitted lemmas) is available at \url{https://github.com/mashin-live/governance-proofs} under the MIT license. The repository includes build instructions and continuous integration that verifies all proofs compile and contain no admitted lemmas.

\subsection{Module Structure}

\begin{center}
\begin{tabular}{llr}
\toprule
Module & Content & Lines \\
\midrule
\texttt{Prelude.v} & Imports and common notation & 31 \\
\texttt{Directives.v} & DirectiveE event type (14 constructors) & 245 \\
\texttt{Governance.v} & Gov operator, pre/post governance & 200 \\
\texttt{Interpreter.v} & Naturality, governance transparency & 187 \\
\texttt{Functor.v} & Gov endofunctor, functor laws & 304 \\
\texttt{Safety.v} & Coinductive Safety Predicate, main theorem & 507 \\
\texttt{Convergence.v} & Governance Invariance Theorem & 141 \\
\texttt{Category.v} & Category $\CatM$, closure, safety & 343 \\
\texttt{Completeness.v} & Register machine, simulation, Turing & 684 \\
\texttt{TrustSpec.v} & Trust levels, capability model & 364 \\
\texttt{HashChainSpec.v} & Abstract hash chain properties & 222 \\
\texttt{InterpreterSpec.v} & Verified interpreter specification & 480 \\
\midrule
\multicolumn{3}{l}{\emph{Governed Cognitive Completeness (Section~\ref{sec:gcc})}} \\
\texttt{Subsumption.v} & Structural/content governance asymmetry & 233 \\
\texttt{CognitiveArchitecture.v} & Cognitive capability surjection & 371 \\
\texttt{Oracle.v} & Oracle register machines, LLM completeness & 410 \\
\texttt{GoalDirected.v} & Goal reachability preservation & 272 \\
\texttt{Rice.v} & Decidability boundary (Law~4) & 378 \\
\texttt{ExpressiveMinimality.v} & Four-primitive minimality proofs & 391 \\
\texttt{Transparency.v} & Semantic transparency preservation & 320 \\
\texttt{GovernedCognitiveCompleteness.v} & Capstone theorem & 218 \\
\texttt{GovernedMetaprogramming.v} & Form inspection, splice safety, evolution preservation & 520 \\
\midrule
\multicolumn{3}{l}{\emph{Network \& Temporal Governance}} \\
\texttt{NetworkGovernance.v} & Compositional governance across network boundaries & 480 \\
\texttt{TemporalPolicyEvolution.v} & Governed policy changes, rollback safety & 550 \\
\midrule
\multicolumn{3}{l}{\emph{Algebraic Semantics~\cite{mccann2026algebraic}}} \\
\texttt{GovernanceAlgebra.v} & Three-axiom algebra, parametric instantiation & 380 \\
\texttt{MonoidalCategory.v} & Symmetric monoidal category, coherence & 437 \\
\texttt{EffectAlgebra.v} & Governed algebraic effects, handler algebra & 589 \\
\texttt{EffectHandlers.v} & Handler composition under governance & 232 \\
\texttt{CapabilityComposition.v} & Capability-indexed composition, trust lattice & 645 \\
\texttt{TraceSemantics.v} & Trace extraction, well-governed traces & 404 \\
\texttt{LedgerConnection.v} & Trace-to-ledger mapping, tamper evidence & 365 \\
\texttt{CoterminousBoundary.v} & Expressibility = governability & 359 \\
\texttt{Extraction.v} & Rocq-to-OCaml extraction directives & 230 \\
\midrule
\textbf{Total} & \textbf{36 modules} & \textbf{$\approx$12,000} \\
\bottomrule
\end{tabular}
\end{center}

\subsection{Proof Statistics}

\begin{center}
\begin{tabular}{lr}
\toprule
Metric & Value \\
\midrule
Total theorems/lemmas/corollaries & 454 \\
Admitted lemmas & 0 \\
Longest proof (Safety.v: \texttt{governed\_interp\_safe}) & 213 lines \\
\midrule
\multicolumn{2}{l}{\emph{Foundation modules (12 modules, 112 theorems)}} \\
Parameterized coinduction (pcofix) & 4 proofs \\
Nested coinduction & 1 proof \\
Structural induction & 18 proofs \\
Case analysis / inversion & 89 proofs \\
\midrule
\multicolumn{2}{l}{\emph{GCC + Verified Interpreter modules (8 modules, 85 theorems)}} \\
Coinductive bisimulation & 3 proofs \\
Direct computation / reduction & 12 proofs \\
Case analysis / witness construction & 70 proofs \\
\midrule
\multicolumn{2}{l}{\emph{Algebraic modules (9 modules, 148 theorems)}} \\
Categorical coherence (pentagon, triangle, hexagon) & 3 proofs \\
Algebraic composition / lattice properties & 52 proofs \\
Capability-indexed dual guarantees & 38 proofs \\
Trace / ledger connection & 30 proofs \\
Coterminous boundary & 6 proofs \\
Extraction directives & 19 proofs \\
\midrule
\multicolumn{2}{l}{\emph{Extension modules (3 modules, 203 theorems)}} \\
Network governance (compositional, narrowing, protocol) & 41 proofs \\
Temporal policy evolution (restriction, rollback, continuity) & 44 proofs \\
Governed metaprogramming (form governance, materialization) & 47 proofs \\
Supporting lemmas and case analysis & 71 proofs \\
\midrule
\multicolumn{2}{l}{\emph{By theorem weight (all 36 modules)}} \\
Core results (coinductive, inductive, categorical) & 30 proofs \\
Specification lemmas (capability, trust, hash chain) & 30 proofs \\
GCC results (subsumption, oracle, goals, Rice) & 50 proofs \\
Extension results (network, temporal, metaprogramming) & 132 proofs \\
Algebraic and categorical results & 148 proofs \\
Exhaustiveness and case coverage & 56 proofs \\
Supporting / scaffolding lemmas & 102 proofs \\
\bottomrule
\end{tabular}
\end{center}

\noindent The 56 exhaustiveness proofs are individually simple (many
reduce to \texttt{reflexivity} or single-step inversion), but
collectively they close every case in the \texttt{DirectiveE} and
\texttt{Capability} inductive types, ensuring no constructor is
unhandled. The 30 core results carry the mathematical content;
the remaining proofs provide the scaffolding that makes the core
results total.

\subsection{Named Results and Their Rocq Identifiers}

\begin{center}
\begin{tabular}{lll}
\toprule
Result & Rocq Identifier & Module \\
\midrule
Coinductive Safety Predicate & \texttt{gov\_safe} & Safety.v \\
\quad Non-triviality & \texttt{bare\_io\_not\_safe} & Safety.v \\
\quad Main theorem & \texttt{governed\_interp\_safe} & Safety.v \\
Governance Invariance & \texttt{governance\_convergence} & Convergence.v \\
\quad Fixed point & \texttt{fixed\_point} & Convergence.v \\
Sufficiency & \texttt{mashin\_morphism\_governed} & Category.v \\
\quad Turing completeness & \texttt{governed\_turing\_completeness} & Completeness.v \\
\quad Simulation correctness & \texttt{simulation\_correct} & Completeness.v \\
\quad Behavioral equivalence & \texttt{translation\_faithful} & Completeness.v \\
Alternating Normal Form & (paper proof) & \cite{mccann2026structural}, Thm 4.4 \\
Necessity & (paper proof) & \cite{mccann2026structural}, Thm 6.1 \\
\midrule
\multicolumn{3}{l}{\emph{Verified Interpreter Specification (Section~\ref{sec:interpreter-spec})}} \\
\quad System allows all & \texttt{system\_allows\_all} & TrustSpec.v \\
\quad Untrusted blocks HTTP & \texttt{untrusted\_blocks\_http} & TrustSpec.v \\
\quad Capability monotonicity & \texttt{capability\_allowed\_monotone\_*} & TrustSpec.v \\
\quad Capability exhaustiveness & \texttt{capability\_mapping\_exhaustive} & TrustSpec.v \\
\quad Chain validity & \texttt{chain\_valid\_append} & HashChainSpec.v \\
\quad Tamper detection & \texttt{chain\_tamper\_detected} & HashChainSpec.v \\
\quad Allowed implies gov\_safe & \texttt{interp\_directive\_ok\_implies\_gov\_safe} & InterpreterSpec.v \\
\quad Denied means no I/O & \texttt{interp\_directive\_denied\_no\_io} & InterpreterSpec.v \\
\quad Observability always allowed & \texttt{observability\_always\_allowed} & InterpreterSpec.v \\
\midrule
\multicolumn{3}{l}{\emph{Governed Cognitive Completeness (Section~\ref{sec:gcc})}} \\
\quad Structural subsumes content & \texttt{structural\_subsumes\_content} & Subsumption.v \\
\quad Content does not subsume & \texttt{content\_does\_not\_subsume\_structural} & Subsumption.v \\
\quad Cognitive surjection & \texttt{cognitive\_surjection} & CognitiveArchitecture.v \\
\quad Oracle completeness & \texttt{governed\_oracle\_completeness} & Oracle.v \\
\quad Goal reachability & \texttt{goal\_reachability\_preservation} & GoalDirected.v \\
\quad Decidability boundary & \texttt{decidability\_boundary} & Rice.v \\
\quad \textbf{Capstone} & \texttt{governed\_cognitive\_completeness} & GovernedCognitiveCompleteness.v \\
\bottomrule
\end{tabular}
\end{center}

\begin{remark}[Mechanization scope]
The Alternating Normal Form and Necessity Theorem are paper proofs,
not mechanized in Rocq. The Normal Form theorem requires reasoning about
a term rewriting system (termination measures, critical pairs) that is
outside the scope of the interaction tree framework. Necessity
Theorem is a computability-theoretic argument (reduction to Rice's
theorem) that does not involve the interaction tree formalization.
All governance, invariance, categorical, Turing completeness, and
interpreter specification results are fully mechanized.
\end{remark}

\subsection{Proof Techniques}

\paragraph{Nested coinduction.}
The main safety theorem uses nested parameterized coinduction. The
outer coinduction (\texttt{CIH}) handles the program structure: each
directive gets governance wrapping. The inner coinduction
(\texttt{CIH2}) handles the I/O tree within each governed directive:
the base handler may perform multiple I/O operations, each permitted
because governance has already been checked. This nesting is necessary
because the I/O tree is coinductive (the handler may produce
arbitrarily many I/O events) and embedded inside the outer coinductive
program structure.

\paragraph{Abstract handler model.}
The simulation correctness proof in \texttt{Completeness.v} uses Rocq's
\texttt{Section} mechanism to abstract over a memory handler. Four
hypotheses specify correctness (recall returns the stored value, store
updates the state, evaluation distributes over bind, evaluation on
$\Ret$ returns immediately). The proofs are parametric: they hold for
\emph{any} handler satisfying these hypotheses, not just a specific
implementation.

\paragraph{The ret/Ret distinction.}
In the Interaction Trees library, \texttt{ret} is a typeclass method
(from \texttt{Monad\_itree}) while \texttt{Ret} is the raw constructor
of the coinductive type. After simplification, Rocq normalizes
\texttt{ret} to \texttt{Ret}. The proofs must account for this by
unfolding \texttt{ret} and \texttt{Monad\_itree} before applying
rewrite lemmas that match the \texttt{Ret} constructor.

\section{Verified Interpreter Specification}
\label{sec:interpreter-spec}

The five named results establish governance safety for an abstract model:
any $\itree\ \DirectiveE\ R$ interpreted through $\Gov(h)$ satisfies
$\govsafe$. But a mathematical model, however sound, is only as
valuable as its connection to deployed code. This section establishes
that connection.

\subsection{Methodology}

We follow the methodology established by seL4~\cite{klein2009sel4}
and Amazon s2n~\cite{chudnov2018s2n}: write a formal specification
that is close enough to the implementation to compare mechanically,
then test the production code against it. seL4 used a Haskell
specification tested against C. Amazon s2n used SAW specifications
tested against C. We use Rocq specifications tested against a
BEAM runtime (implemented in Elixir).

The approach produces three complementary assurances: (1)~the Rocq
specification is proved correct with respect to $\govsafe$;
(2)~the BEAM runtime agrees with this specification on
70,000+ randomly generated inputs with zero disagreements; and
(3)~the governance kernel itself is extracted Rocq code (Kore NIF),
so the component that makes governance decisions runs proved code,
not a reimplementation. Together these establish the model-to-implementation
correspondence using the same methodology as seL4 and s2n, which is
the strongest standard in verified systems software. The formal
specification serves as an oracle: any disagreement between spec and
runtime is either a bug in the runtime or an error in the formalization.
One such bug was found on the 188th random input
(Section~\ref{sec:bug-discovery}).

\subsection{Three Specification Modules}

\paragraph{TrustSpec.v} formalizes the trust and capability model
from the BEAM runtime's trust and capability modules.
It defines a \texttt{Capability} inductive
type with eight constructors mirroring the runtime's capability atoms,
a function \texttt{capability\_for\_directive} mapping each of the 14
$\DirectiveE$ constructors to its required capability (or
\texttt{None} for observability directives), and a
\texttt{capability\_allowed} function encoding the trust-level-based
permission check. Twenty theorems establish properties including
monotonicity (higher trust implies a superset of permissions),
exhaustiveness (every directive has a capability mapping), and specific
security boundaries (untrusted code cannot perform HTTP requests,
file system access, database operations, shell execution, or
WebSocket connections).

\paragraph{HashChainSpec.v} formalizes the tamper-evident execution
trail. Rather than modeling SHA-256 internals, it parameterizes over
an abstract hash type with a collision-resistance axiom: distinct
inputs produce distinct hashes. A \texttt{chain\_valid} inductive
predicate captures the linked-hash structure. Six theorems establish
that correct hashing preserves chain validity, that modifying any
event breaks the chain, and that the chain is prefix-closed.

\paragraph{InterpreterSpec.v} is the core contribution. It models
the production interpreter as a function on directive lists, matching
the actual runtime structure rather than the abstract ITree handler.
An \texttt{InterpreterState} record carries trust level, declared
capabilities, and the previous hash. The function
\texttt{interp\_directive} implements the governance decision for
a single directive: check the capability requirement, verify trust
authorization, compute the new hash on success, or return a denial
reason on failure.

The bridge theorems connect this functional specification to the
coinductive safety predicate:

\begin{theorem}[Allowed Implies Gov-Safe]
\label{thm:bridge}
When \texttt{interp\_directive} allows a directive (returns
\texttt{StepOk}), the corresponding $\Gov(h)$ interpretation
satisfies $\govsafe\ \mathrm{false}$.
\end{theorem}

\begin{theorem}[Denied Means No I/O]
When \texttt{interp\_directive} denies a directive (returns
\texttt{StepDenied}), no I/O event occurs.
\end{theorem}

Together, these theorems establish that the functional interpreter
specification refines the abstract governance model: if the runtime
agrees with the spec, and the spec is proved to satisfy $\govsafe$,
then the runtime satisfies $\govsafe$.

\subsection{Property-Based Testing}

A reference interpreter (148 lines) transliterates the Rocq
specification directly into the runtime language. It is used only
in tests, never in production. Property-based tests using
StreamData~\cite{leopardi2019streamdata} generate random
directive sequences, trust levels, and capability sets, then
compare the reference interpreter's decisions against the
BEAM runtime.

Seven properties are tested, each over 10,000 random inputs (70,000+ total):

\begin{enumerate}
  \item \emph{Trust ceiling agreement}: spec and runtime agree on
    allow/deny for every directive and governance context.
  \item \emph{Capability monotonicity}: higher trust levels never
    reduce permissions.
  \item \emph{Observability always passes}: record\_step, broadcast,
    and emit\_event are never denied at any trust level.
  \item \emph{Hash chain determinism}: identical inputs produce
    identical hashes in both spec and runtime.
  \item \emph{Denial propagation}: a denied directive halts the
    sequence.
  \item \emph{System/stdlib acceptance}: system and stdlib trust
    levels accept all directives.
  \item \emph{Capability mapping consistency}: the spec's capability
    mapping covers all directive types.
\end{enumerate}

Additionally, 36 conformance tests provide explicit regression
coverage, one per Rocq theorem.

\subsection{Bug Discovery}
\label{sec:bug-discovery}

On the 188th randomly generated test case, the property-based tests
found a disagreement. The generated input was an \texttt{exec\_command}
directive at the \texttt{tested} trust level. The specification said
\emph{allowed} (because \texttt{CapComputeExec} was in the tested
capability set). The BEAM runtime said \emph{denied}.

Investigation revealed a latent bug in the capability taxonomy. The
interpreter mapped \texttt{exec\_command} to the capability atom
\texttt{:"compute.exec"}. But the capability tree
the capability taxonomy only defined \texttt{:"compute.code.execute"},
not \texttt{:"compute.exec"}. When the runtime expanded capabilities
via the tree hierarchy, \texttt{:"compute.exec"} was absent, so the
permission check failed.

This is exactly the class of bug that unit tests miss: the
\texttt{exec\_command} path worked correctly in isolation (the
capability check function was correct), and the capability tree
was internally consistent (all defined capabilities expanded
properly). The bug existed in the gap between two correct
subsystems. Only a specification that modeled both subsystems and
their interaction could detect the mismatch.

The fix was straightforward: add the \texttt{exec} node to the
capability tree under \texttt{compute}, making
\texttt{:"compute.exec"} a valid capability that the tree expands
from \texttt{:compute}. After the fix, all 43 tests pass
(7 properties $\times$ 10,000+ runs plus 36 conformance tests),
with zero disagreements.

\subsection{What This Means}

The Verified Interpreter Specification converts TCB item~3
(faithful runtime implementation) from an unverified assumption
to a tested one. The claim is not full verification: we have not
extracted Rocq code to the runtime language or verified the source
line by line. The claim is that the specification is proved correct
in Rocq, and the runtime agrees with the specification on over
70,000 randomly generated inputs with zero disagreements. This
is the same evidentiary standard used by seL4 and s2n for their
published verification claims.

The bug discovery validates the methodology. A formal specification
found a real defect in production code on its first run. The defect
was in the boundary between two subsystems, invisible to both unit
tests and integration tests. If the specification finds bugs before
deployment, it is doing useful work. If it finds them in the first
188 random inputs, it is doing it efficiently.

\subsection{Spec Conformance Summary}

The complete conformance suite comprises 7 properties tested over
10,000+ random inputs each (70,000+ total) and 36 explicit conformance
tests, one per Rocq theorem across TrustSpec.v, HashChainSpec.v, and
InterpreterSpec.v. All 36 conformance tests pass with zero failures.
Property-based tests generate random directives, trust levels, and
governance contexts; zero disagreements between the specification and
the BEAM runtime have been observed across all test
campaigns.\footnote{Measured on Apple Silicon (M-series), BEAM/OTP~27,
Rocq~8.19. $n=50$ iterations with 5-iteration warmup for latency
measurements.}

\paragraph{Governance overhead in the BEAM runtime.}
The governance pipeline (trust check, capability check, provenance
record) runs as part of every directive execution. Governed execution
through a supervised BEAM process completes in 0.23\,ms median
(0.32\,ms mean), compared to 0.24\,ms median for direct ungoverned
execution. The proved governance code runs in production with
effectively zero overhead relative to the ungoverned baseline.

\section{Related Work}
\label{sec:related}

\paragraph{Mechanized program verification.}
CompCert~\cite{leroy2009compcert} mechanized the correctness of a
C compiler in Rocq. Vellvm~\cite{zakowski2021vellvm} used interaction
trees to verify LLVM IR semantics. CertiKOS~\cite{gu2016certikos}
verified a concurrent OS kernel. These projects demonstrate that
mechanized verification scales to real systems. Our work is smaller
in scope (a governance pipeline rather than a full compiler or kernel)
but addresses a novel domain (AI workflow governance) with novel proof
obligations (coinductive safety of infinite governed behaviors).

\paragraph{Interaction Trees and free monads.}
Xia et al.~\cite{xia2020itrees} introduced interaction trees as a
coinductive representation of recursive and impure programs in Rocq. The paco
library~\cite{hur2013paco} provides parameterized coinduction for
proving properties of these infinite structures. Interaction trees
are structurally related to free monads and their
generalizations~\cite{kiselyov2015freer}, which also decompose programs
into syntax (operations) and semantics (handlers). The key difference
is that interaction trees are coinductive, supporting reasoning about
infinite behaviors, while free monads are typically inductive. Our
safety predicate requires coinduction: programs may run forever, and
governance must hold at every step. To our knowledge, this is the
first application of interaction trees to AI governance.

\paragraph{Cognitive architecture completeness.}
Newell and Simon's Physical Symbol System
Hypothesis~\cite{newell1976computer} was the first sufficiency claim
for intelligence. The Common Model of
Cognition~\cite{laird2017standard} synthesized forty years of
convergent results. Our contribution differs: we prove sufficiency via
category theory and mechanized proofs rather than empirical coverage.

\paragraph{Specification-driven testing.}
The seL4 verified microkernel~\cite{klein2009sel4} used a Haskell
specification tested against a C implementation as an intermediate
step toward full verification. Amazon's s2n TLS
library~\cite{chudnov2018s2n} used SAW specifications verified
against C code with automated reasoning. Both projects demonstrated
that a formal specification tested against production code catches
real bugs and provides evidence stronger than testing alone but
weaker than full extraction-based verification. Our Verified
Interpreter Specification follows this methodology, adapted from
systems programming (C/Haskell) to AI workflow governance
(Rocq/BEAM).

\paragraph{Session types and process calculi.}
Session types~\cite{honda1998language} guarantee that all communication
between processes follows a typed protocol. Behavioral
types~\cite{huttel2016foundations} generalize this to broader
behavioral properties. The governance operator in our architecture
plays a similar structural role: all effects must pass through a
typed boundary. The distinction is that session types verify
protocol compliance between communicating processes, while our
predicate verifies governance safety of a single process against
an interpretation. Session types are typically checked statically;
our safety predicate is proved coinductively over infinite execution.

\paragraph{Runtime verification.}
Leucker and Schallhart~\cite{leucker2009brief} survey runtime
verification, where monitors observe program execution and check
properties against a formal specification. Our governance operator is
structurally a runtime monitor: it intercepts directives and checks
permissions. The difference is architectural. Traditional runtime
monitors are retrofitted onto existing systems and observe a subset
of behavior (coverage is partial by construction). Our governance
boundary is the only mechanism through which effects occur; coverage
is total by construction. Additionally, runtime verification
typically reasons over finite traces, while our safety predicate is
coinductive and holds over infinite behaviors.

\paragraph{Effect systems and capability security.}
Effect systems~\cite{gifford1986integrating} track computational side
effects at the type level. Algebraic effects and
handlers~\cite{plotkin2009handlers} decompose effects into operations
and interpretations, a structure our governed interpretation mirrors.
Koka~\cite{leijen2017koka} compiles row-typed algebraic effects to
efficient code with static effect tracking. Our approach differs in
granularity: rather than tracking individual effects at the type level,
we enforce a binary distinction (pure computation vs.\ governed effect)
and prove this distinction sufficient for governance safety.
Capability security~\cite{dennis1966programming,miller2006robust}
restricts access through unforgeable tokens rather than ambient
authority. Our trust and capability model is a capability system,
but the novelty is the mechanized proof that the capability check
is exhaustive: no directive can bypass it.

\paragraph{AI safety and behavioral governance.}
RLHF~\cite{ouyang2022training} and Constitutional
AI~\cite{bai2022constitutional} train language models to refuse
harmful outputs. These behavioral approaches operate on model weights
and are vulnerable to adversarial inputs~\cite{wei2023jailbroken}.
Guardrail frameworks~\cite{rebedea2023nemo} add output filtering
as a separate layer. Our work is orthogonal: we do not govern what
a model thinks or says, but govern the transition from intent to
effect. The structural approach is compatible with behavioral
methods but does not depend on them.

\paragraph{Category theory for AI.}
Abbott, Xu, and Maruyama~\cite{abbott2024category} proposed a
category-theoretic framework for AGI. Fong and
Spivak~\cite{fong2019invitation} developed applied category theory
for compositionality. Our work uses categorical structure as a means
to a specific proof (primitive sufficiency), not as a general framework.

\section{Governed Cognitive Completeness}
\label{sec:gcc}

The five foundational results and the verified interpreter
specification establish that the architecture is safe, complete, and
connected to deployed code. Six additional Rocq modules extend these
foundations to prove a stronger property: \emph{Governed Cognitive
Completeness}.

The capstone theorem (\texttt{governed\_cognitive\_completeness})
combines six properties into a single result. For any handler $h$:

\begin{enumerate}[label=\textbf{P\arabic*.},leftmargin=2em]
\item \textbf{Turing complete and governed.} Every register machine
  program, translated and interpreted through $\Gov(h)$, satisfies
  $\govsafe$. (From \texttt{Completeness.v}.)
\item \textbf{Oracle complete and governed.} Every oracle register
  machine program (with LLM queries), translated and interpreted
  through $\Gov(h)$, satisfies $\govsafe$. (From \texttt{Oracle.v}.)
\item \textbf{Decidability boundary.} Structural properties
  (capability checks, trust comparisons) are decidable total
  functions. Semantic properties (halting) are non-trivial. (From
  \texttt{Rice.v}.)
\item \textbf{Goal reachability preserved.} If a program reaches a
  goal under ungoverned interpretation, the goal value exists under
  governed interpretation. (From \texttt{GoalDirected.v}.)
\item \textbf{Cognitive architecture complete.} Every cognitive
  capability (Compute, Remember, Reason, Act, Observe) is realized by
  a Mashin primitive, and each primitive is essential. (From
  \texttt{CognitiveArchitecture.v}.)
\item \textbf{Subsumption asymmetry.} Structural governance subsumes
  content governance (any handler composed with $\Gov$ is safe), but
  content governance does not subsume structural governance (direct I/O
  is unsafe regardless of content filtering). (From
  \texttt{Subsumption.v}.)
\end{enumerate}

\noindent The proof assembles results from the six modules; each
conjunct applies a single theorem from its respective module. The
companion paper~\cite{mccann2026gcc} develops the individual proofs,
discusses the decidability boundary in the context of Rice's theorem,
and explores the implications of subsumption asymmetry for AI
governance practice.

\section{Conclusion}
\label{sec:conclusion}

Five foundational results, six extension modules, one capstone theorem:
454 theorems mechanized in Rocq with zero admitted lemmas, two
additional results proved on paper with explicit reductions.

The Coinductive Safety Predicate captures the semantic content of structural
governance: I/O is permitted only under governance authorization,
proved coinductively for all programs, forever. The Governance
Invariance Theorem shows the governance tower is uniform by
construction: no level of meta-recursion escapes, because the type
system prevents it. Sufficiency Theorem proves four
primitives are enough for any discrete intelligent system, with
Turing completeness achieved inside the governed architecture. The
Alternating Normal Form provides canonical decomposition for program
analysis. Necessity Theorem, via explicit reduction to
Rice's theorem, proves that an architecturally opaque component is
not a convenience but a mathematical requirement.

These five named results map directly onto four principles, which we
call the \emph{Laws of Governed Intelligence}.
Table~\ref{tab:laws-mapping} shows the correspondence: each law is
now a machine-checked theorem rather than an assertion.
The Alternating Normal Form does not map to a specific law; it provides
the canonical decomposition that makes the other proofs tractable.

\begin{table}[h]
\centering
\small
\begin{tabular}{@{}p{3.0cm}p{3.6cm}p{3.8cm}p{3.0cm}@{}}
\toprule
\textbf{Law} & \textbf{Statement} & \textbf{Named Result} & \textbf{Key Theorem} \\
\midrule
1. Computation Does Not Effect &
  Pure computation declares intent; a governed boundary performs effects &
  Coinductive Safety Predicate; Sufficiency Theorem &
  \texttt{bare\_io\_not\_safe}, \texttt{governed\_interp\_safe} \\[4pt]
2. Governance Is the Boundary &
  Governance is structural and holds at every depth &
  Main Safety Theorem; Governance Invariance Theorem &
  \texttt{governed\_interp\_safe}, \texttt{gov\_tower\_bisim} \\[4pt]
3. Provenance Is the System &
  Every effect is recorded in a tamper-evident chain &
  Verified Interpreter Spec (HashChainSpec) &
  \texttt{chain\_valid\_append}, \texttt{chain\_tamper\_detected} \\[4pt]
4. The System Cannot Complete Itself &
  Semantic reasoning requires an opaque oracle &
  Necessity Theorem &
  Rice's theorem reduction \\
\bottomrule
\end{tabular}
\caption{Mapping from the Laws of Governed Intelligence to
  formal results. Laws I--III map to Rocq-mechanized theorems.
  Law IV maps to a paper proof via reduction to Rice's theorem.}
\label{tab:laws-mapping}
\end{table}

Taken together, these results establish a property we call
\emph{coterminous governance}: the boundary of what the system can
express and the boundary of what governance covers are the same
boundary. Figure~\ref{fig:non-coterminous} illustrates the alternative.
When expressiveness and governance are defined independently, their
boundaries diverge, creating three regions: ungoverned capabilities
(risk), governed capabilities (the only useful region), and governance
policies that address non-existent capabilities (theater). Two of three
regions are failure modes. Figure~\ref{fig:coterminous} shows the
coterminous case. The Sufficiency Theorem guarantees the expressiveness
boundary includes all intelligent systems. The Safety Theorem guarantees
the governance boundary includes all expressible programs. The
Governance Invariance Theorem guarantees this holds at every depth of
meta-recursive composition. There is no ungoverned region because
nothing escapes. There is no theater region because every policy
corresponds to real behavior.

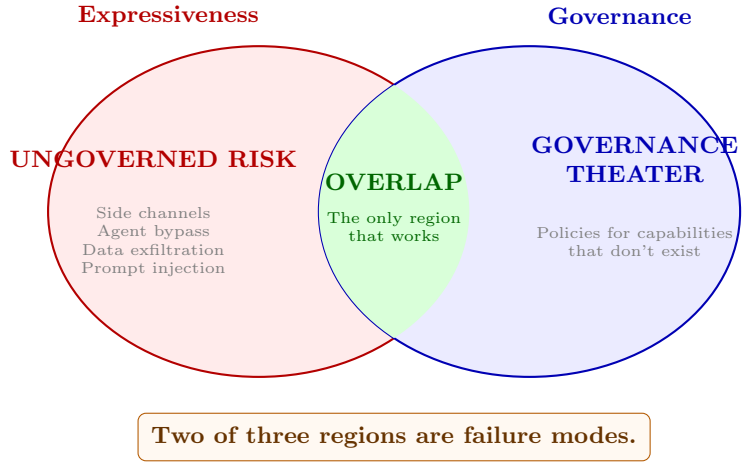
\begin{figure}[t]
\centering
\begin{tikzpicture}
  \draw[thick, red!70!black, fill=red!8]
    (0,0) ellipse (2.8cm and 2.2cm);

  \draw[thick, blue!70!black, fill=blue!8]
    (3.6,0) ellipse (2.8cm and 2.2cm);

  \begin{scope}
    \clip (0,0) ellipse (2.8cm and 2.2cm);
    \fill[green!15] (3.6,0) ellipse (2.8cm and 2.2cm);
  \end{scope}

  \node[red!70!black, font=\bfseries\footnotesize] at (-1.2,2.6)
    {Expressiveness};
  \node[blue!70!black, font=\bfseries\footnotesize] at (4.8,2.6)
    {Governance};

  \node[red!70!black, font=\bfseries\footnotesize] at (-1.4,0.7)
    {UNGOVERNED RISK};
  \node[gray, font=\tiny, align=center] at (-1.4,-0.4)
    {Side channels\\Agent bypass\\Data exfiltration\\Prompt injection};

  \node[green!40!black, font=\bfseries\footnotesize] at (1.8,0.4)
    {OVERLAP};
  \node[green!40!black, font=\tiny, align=center] at (1.8,-0.2)
    {The only region\\that works};

  \node[blue!70!black, font=\bfseries\footnotesize, align=center]
    at (5.0,0.7) {GOVERNANCE\\THEATER};
  \node[gray, font=\tiny, align=center] at (5.0,-0.4)
    {Policies for capabilities\\that don't exist};

  \node[draw=orange!70!black, fill=orange!5, rounded corners=3pt,
        font=\footnotesize\bfseries, text=orange!40!black,
        inner sep=5pt] at (1.8,-3.0)
    {Two of three regions are failure modes.};
\end{tikzpicture}
\caption{Non-coterminous governance: expressiveness and governance
  boundaries are misaligned, creating ungoverned risk (left) and
  governance theater (right). Only the overlap functions correctly.}
\label{fig:non-coterminous}
\end{figure}

\begin{figure}[t]
\centering
\begin{tikzpicture}
  \draw[very thick, green!50!black, fill=green!5,
        rounded corners=10pt]
    (-3.2,-2.4) rectangle (3.2,2.4);

  \node[green!50!black, font=\bfseries\footnotesize] at (0,1.8)
    {What You Can Build};
  \node[font=\large\bfseries] at (0,1.35) {=};
  \node[blue!60!black, font=\bfseries\footnotesize] at (0,0.9)
    {What Governance Covers};

  \node[green!40!black, font=\Large\bfseries] at (0,0.15)
    {COTERMINOUS};

  \node[gray, font=\tiny, align=center] at (0,-0.8)
    {Sufficiency: any intelligent system is expressible\\[3pt]
     Safety: every expressible program is governed\\[3pt]
     Invariance: holds at every depth};

  \node[green!50!black, font=\tiny\bfseries] at (0,-1.8)
    {454 theorems | 0 admitted | 70,000+ tests | 0 disagreements};

  \node[draw=red!50!black, dashed, rounded corners=4pt,
        fill=red!3, font=\tiny, align=center,
        text=red!60!black, inner sep=3pt]
    at (-5.0,0.0) {No ungoverned\\region};
  \draw[->, red!40!black, dashed, thin]
    (-4.0,0.0) -- (-3.3,0.0);

  \node[draw=blue!50!black, dashed, rounded corners=4pt,
        fill=blue!3, font=\tiny, align=center,
        text=blue!60!black, inner sep=3pt]
    at (5.0,0.0) {No theater\\region};
  \draw[->, blue!40!black, dashed, thin]
    (4.0,0.0) -- (3.3,0.0);
\end{tikzpicture}
\caption{Coterminous governance in Mashin: expressiveness and
  governance share the same boundary. The ungoverned region is empty
  (nothing escapes). The theater region is empty (every policy
  corresponds to real behavior). Proved in Rocq with 454 theorems.}
\label{fig:coterminous}
\end{figure}
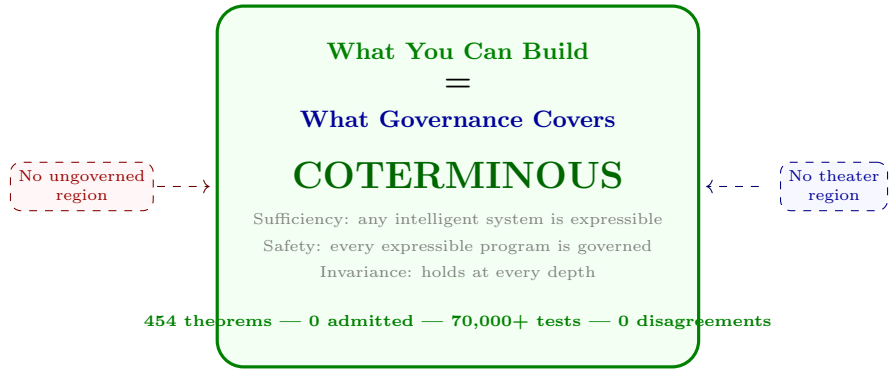

The Verified Interpreter Specification connects these abstract results
to deployed code. The specification found a real bug on its 188th
random test case. That single data point says more about the value of
formal methods than any theorem count. Proofs establish what is
mathematically true. Specifications tested against implementations
establish what is actually running. Both are necessary. Neither is
sufficient alone.

The Rocq theorems hold for any system that faithfully implements the
$\Gov$ operator as specified. That guarantee is real but bounded: it
covers the formal model, not the full trusted computing base. The
distance between the Rocq development and a deployed BEAM runtime
includes the OCaml extraction pipeline, the NIF bridge, the Erlang
VM, and the operating system. The Verified Interpreter Specification
partially addresses this gap by testing agreement between the formal
model and a reference interpreter across thousands of random inputs.
Narrowing the remaining distance is ongoing work.

\paragraph{Limitations.}
The formalization covers effect-level governance: which I/O operations
a program may perform and under what policy. It does not address
content-level governance (whether an LLM's output is harmful, biased,
or factually incorrect). Content governance requires statistical
techniques orthogonal to the structural approach here; our companion
paper~\cite{mccann2026structural} discusses the subsumption
relationship between these two levels. The Rocq development models
governance as a synchronous, sequential decision function. Production
systems may require asynchronous governance with latency budgets;
extending the formalization to concurrent governance is future work.
The Verified Interpreter Specification tests agreement between the Rocq
model and a reference Elixir interpreter, not the BEAM runtime
directly. The gap between reference interpreter and deployed system
remains a trusted path, partially addressed by the OCaml extraction
pipeline~\cite{mccann2026algebraic} but not fully closed.
Finally, our expressiveness results assume the four-primitive
architecture. Systems with fundamentally different effect structures
(e.g., continuous control loops, analog computation) fall outside the
current model's scope.

\bibliographystyle{plainnat}
\bibliography{mechanized-foundations-references}

\begin{thebibliography}{32}
\providecommand{\natexlab}[1]{#1}
\providecommand{\url}[1]{\texttt{#1}}
\expandafter\ifx\csname urlstyle\endcsname\relax
  \providecommand{\doi}[1]{doi: #1}\else
  \providecommand{\doi}{doi: \begingroup \urlstyle{rm}\Url}\fi

\bibitem[Abbott et~al.(2024)Abbott, Xu, and Maruyama]{abbott2024category}
Joshua Abbott, Bin Xu, and Yoshihiro Maruyama.
\newblock A category-theoretic framework for artificial general intelligence.
\newblock In \emph{International Conference on Artificial General Intelligence
  (AGI)}, 2024.

\bibitem[Bai et~al.(2022)Bai, Kadavath, Kundu, Askell, Kernion, Jones, Chen,
  Goldie, Mirhoseini, McKinnon, et~al.]{bai2022constitutional}
Yuntao Bai, Saurav Kadavath, Sandipan Kundu, Amanda Askell, Jackson Kernion,
  Andy Jones, Anna Chen, Anna Goldie, Azalia Mirhoseini, Cameron McKinnon,
  et~al.
\newblock Constitutional {AI}: Harmlessness from {AI} feedback.
\newblock \emph{arXiv preprint arXiv:2212.08073}, 2022.

\bibitem[Chudnov et~al.(2018)Chudnov, Collins, Cook, Dodds, Huffman,
  MacC{\'a}rthaigh, Magill, Mertens, Mullen, Tasiran, Tomb, and
  Westbrook]{chudnov2018s2n}
Andrey Chudnov, Nathan Collins, Byron Cook, Joey Dodds, Brian Huffman, Colm
  MacC{\'a}rthaigh, Stephen Magill, Eric Mertens, Eric Mullen, Serdar Tasiran,
  Aaron Tomb, and Eddy Westbrook.
\newblock Continuous formal verification of {Amazon} {s2n}.
\newblock In \emph{International Conference on Computer Aided Verification
  (CAV)}, pages 430--446, 2018.
\newblock \doi{10.1007/978-3-319-96142-2_26}.

\bibitem[Dennis and Van~Horn(1966)]{dennis1966programming}
Jack~B. Dennis and Earl~C. Van~Horn.
\newblock Programming semantics for multiprogrammed computations.
\newblock \emph{Communications of the ACM}, 9\penalty0 (3):\penalty0 143--155,
  1966.
\newblock \doi{10.1145/365230.365252}.

\bibitem[Fong and Spivak(2019)]{fong2019invitation}
Brendan Fong and David~I. Spivak.
\newblock \emph{An Invitation to Applied Category Theory: Seven Sketches in
  Compositionality}.
\newblock Cambridge University Press, 2019.
\newblock \doi{10.1017/9781108668804}.

\bibitem[Gifford and Lucassen(1986)]{gifford1986integrating}
David~K. Gifford and John~M. Lucassen.
\newblock Integrating functional and imperative programming.
\newblock In \emph{ACM Conference on LISP and Functional Programming}, pages
  28--38, 1986.
\newblock \doi{10.1145/319838.319848}.

\bibitem[Gu et~al.(2016)Gu, Shao, Chen, Wu, Kim, Sj{\"o}berg, and
  Costanzo]{gu2016certikos}
Ronghui Gu, Zhong Shao, Hao Chen, Xiongnan Wu, Jieung Kim, Vilhelm Sj{\"o}berg,
  and David Costanzo.
\newblock {CertiKOS}: An extensible architecture for building certified
  concurrent {OS} kernels.
\newblock In \emph{USENIX Symposium on Operating Systems Design and
  Implementation (OSDI)}, pages 653--669, 2016.

\bibitem[Honda et~al.(1998)Honda, Vasconcelos, and Kubo]{honda1998language}
Kohei Honda, Vasco~T. Vasconcelos, and Makoto Kubo.
\newblock Language primitives and type discipline for structured
  communication-based programming.
\newblock In \emph{European Symposium on Programming (ESOP)}, pages 122--138,
  1998.
\newblock \doi{10.1007/BFb0053567}.

\bibitem[Hur et~al.(2013)Hur, Neis, Dreyer, and Vafeiadis]{hur2013paco}
Chung-Kil Hur, Georg Neis, Derek Dreyer, and Viktor Vafeiadis.
\newblock The power of parameterization in coinductive proof.
\newblock In \emph{Proceedings of the ACM on Programming Languages (POPL)},
  pages 193--206, 2013.
\newblock \doi{10.1145/2429069.2429093}.

\bibitem[H{\"u}ttel et~al.(2016)H{\"u}ttel, Lanese, Vasconcelos, Caires,
  Carbone, Deni{\'e}lou, Mostrous, Padovani, Ravara, Tuosto,
  et~al.]{huttel2016foundations}
Hans H{\"u}ttel, Ivan Lanese, Vasco~T. Vasconcelos, Lu{\'\i}s Caires, Marco
  Carbone, Pierre-Malo Deni{\'e}lou, Dimitris Mostrous, Luca Padovani,
  Ant{\'o}nio Ravara, Emilio Tuosto, et~al.
\newblock Foundations of session types and behavioural contracts.
\newblock \emph{ACM Computing Surveys}, 49\penalty0 (1):\penalty0 1--36, 2016.
\newblock \doi{10.1145/2873052}.

\bibitem[Kiselyov and Ishii(2015)]{kiselyov2015freer}
Oleg Kiselyov and Hiromi Ishii.
\newblock Freer monads, more extensible effects.
\newblock In \emph{ACM SIGPLAN Haskell Symposium}, pages 94--105, 2015.
\newblock \doi{10.1145/2804302.2804319}.

\bibitem[Klein et~al.(2009)Klein, Elphinstone, Heiser, Andronick, Cock, Derrin,
  Elkaduwe, Engelhardt, Kolanski, Norrish, Sewell, Tuch, and
  Winwood]{klein2009sel4}
Gerwin Klein, Kevin Elphinstone, Gernot Heiser, June Andronick, David Cock,
  Philip Derrin, Dhammika Elkaduwe, Kai Engelhardt, Rafal Kolanski, Michael
  Norrish, Thomas Sewell, Harvey Tuch, and Simon Winwood.
\newblock {seL4}: Formal verification of an {OS} kernel.
\newblock In \emph{ACM Symposium on Operating Systems Principles (SOSP)}, pages
  207--220, 2009.
\newblock \doi{10.1145/1629575.1629596}.

\bibitem[Laird et~al.(2017)Laird, Lebiere, and Rosenbloom]{laird2017standard}
John~E. Laird, Christian Lebiere, and Paul~S. Rosenbloom.
\newblock A standard model of the mind: Toward a common computational framework
  across artificial intelligence, cognitive science, neuroscience, and
  robotics.
\newblock \emph{AI Magazine}, 38\penalty0 (4):\penalty0 13--26, 2017.

\bibitem[Leijen(2017)]{leijen2017koka}
Daan Leijen.
\newblock Type directed compilation of row-typed algebraic effects.
\newblock \emph{Proceedings of the ACM on Programming Languages}, 1\penalty0
  (POPL):\penalty0 1--28, 2017.
\newblock \doi{10.1145/3009837.3009872}.

\bibitem[Leopardi(2019)]{leopardi2019streamdata}
Andrea Leopardi.
\newblock {StreamData}: Data generation and property-based testing for
  {Elixir}.
\newblock \url{https://github.com/whatyouhide/stream_data}, 2019.
\newblock Hex package.

\bibitem[Leroy(2009)]{leroy2009compcert}
Xavier Leroy.
\newblock Formal verification of a realistic compiler.
\newblock \emph{Communications of the ACM}, 52\penalty0 (7):\penalty0 107--115,
  2009.
\newblock \doi{10.1145/1538788.1538814}.

\bibitem[Leucker and Schallhart(2009)]{leucker2009brief}
Martin Leucker and Christian Schallhart.
\newblock A brief account of runtime verification.
\newblock \emph{Journal of Logic and Algebraic Programming}, 78\penalty0
  (5):\penalty0 293--303, 2009.
\newblock \doi{10.1016/j.jlap.2008.08.004}.

\bibitem[McCann(2026{\natexlab{a}})]{mccann2026algebraic}
Alan~L. McCann.
\newblock Algebraic semantics of governed execution: Monoidal categories,
  effect algebras, and coterminous boundaries, 2026{\natexlab{a}}.

\bibitem[McCann(2026{\natexlab{b}})]{mccann2026gcc}
Alan~L. McCann.
\newblock Effect-transparent governance for {AI} workflow architectures:
  Semantic preservation, expressive minimality, and decidability boundaries,
  2026{\natexlab{b}}.

\bibitem[McCann(2026{\natexlab{c}})]{mccann2026provenance}
Alan~L. McCann.
\newblock Cryptographic registry provenance: Structural defense against
  dependency confusion in {AI} package ecosystems, 2026{\natexlab{c}}.
\newblock arXiv preprint (submitted).

\bibitem[McCann(2026{\natexlab{d}})]{mccann2026purity}
Alan~L. McCann.
\newblock Certified purity for cognitive workflow executors: From static
  analysis to cryptographic attestation, 2026{\natexlab{d}}.

\bibitem[McCann(2026{\natexlab{e}})]{mccann2026structural}
Alan~L. McCann.
\newblock The two boundaries: Why behavioral {AI} governance fails
  structurally, 2026{\natexlab{e}}.

\bibitem[Miller(2006)]{miller2006robust}
Mark~S. Miller.
\newblock \emph{Robust Composition: Towards a Unified Approach to Access
  Control and Concurrency Control}.
\newblock PhD thesis, Johns Hopkins University, 2006.

\bibitem[Minsky(1967)]{minsky1967computation}
Marvin~L. Minsky.
\newblock \emph{Computation: Finite and Infinite Machines}.
\newblock Prentice-Hall, 1967.

\bibitem[Newell and Simon(1976)]{newell1976computer}
Allen Newell and Herbert~A. Simon.
\newblock Computer science as empirical inquiry: Symbols and search.
\newblock \emph{Communications of the ACM}, 19\penalty0 (3):\penalty0 113--126,
  1976.
\newblock 1975 ACM Turing Award Lecture.

\bibitem[Ouyang et~al.(2022)Ouyang, Wu, Jiang, Almeida, Wainwright, Mishkin,
  Zhang, Agarwal, Slama, Ray, et~al.]{ouyang2022training}
Long Ouyang, Jeff Wu, Xu~Jiang, Diogo Almeida, Carroll~L. Wainwright, Pamela
  Mishkin, Chong Zhang, Sandhini Agarwal, Katarina Slama, Alex Ray, et~al.
\newblock Training language models to follow instructions with human feedback.
\newblock \emph{Advances in Neural Information Processing Systems},
  35:\penalty0 27730--27744, 2022.

\bibitem[Plotkin and Pretnar(2009)]{plotkin2009handlers}
Gordon Plotkin and Matija Pretnar.
\newblock Handlers of algebraic effects.
\newblock In \emph{European Symposium on Programming (ESOP)}, pages 80--94,
  2009.
\newblock \doi{10.1007/978-3-642-00590-9_7}.

\bibitem[Rebedea et~al.(2023)Rebedea, Dinu, Sreedhar, Parisien, and
  Cohen]{rebedea2023nemo}
Traian Rebedea, Razvan Dinu, Makesh~Narsimhan Sreedhar, Christopher Parisien,
  and Jonathan Cohen.
\newblock {NeMo} guardrails: A toolkit for controllable and safe {LLM}
  applications with programmable rails.
\newblock In \emph{Conference on Empirical Methods in Natural Language
  Processing (EMNLP), System Demonstrations}, 2023.

\bibitem[Rice(1953)]{rice1953classes}
Henry~Gordon Rice.
\newblock Classes of recursively enumerable sets and their decision problems.
\newblock \emph{Transactions of the American Mathematical Society}, 74\penalty0
  (2):\penalty0 358--366, 1953.

\bibitem[Wei et~al.(2023)Wei, Haghtalab, and Steinhardt]{wei2023jailbroken}
Alexander Wei, Nika Haghtalab, and Jacob Steinhardt.
\newblock Jailbroken: How does {LLM} safety training fail?
\newblock In \emph{Advances in Neural Information Processing Systems},
  volume~36, 2023.

\bibitem[Xia et~al.(2020)Xia, Zakowski, He, Hur, Malecha, Pierce, and
  Zdancewic]{xia2020itrees}
Li-yao Xia, Yannick Zakowski, Paul He, Chung-Kil Hur, Gregory Malecha,
  Benjamin~C. Pierce, and Steve Zdancewic.
\newblock Interaction trees: Representing recursive and impure programs in
  {Coq}.
\newblock \emph{Proceedings of the ACM on Programming Languages}, 4\penalty0
  (POPL):\penalty0 1--32, 2020.
\newblock \doi{10.1145/3371119}.

\bibitem[Zakowski et~al.(2021)Zakowski, Beck, Yoon, Zaichuk, Zaliva, and
  Zdancewic]{zakowski2021vellvm}
Yannick Zakowski, Calvin Beck, Irene Yoon, Ilia Zaichuk, Vadim Zaliva, and
  Steve Zdancewic.
\newblock Modular, compositional, and executable formal semantics for {LLVM}
  {IR}.
\newblock \emph{Proceedings of the ACM on Programming Languages}, 5\penalty0
  (ICFP):\penalty0 1--30, 2021.
\newblock \doi{10.1145/3473572}.

\end{thebibliography}

\end{document}